\documentclass[11pt]{article}

\usepackage[final]{acl}

\usepackage{times}
\usepackage{latexsym}
\usepackage{thmtools} 
\usepackage{thm-restate}
\usepackage[most]{tcolorbox}
\usepackage{enumitem}
\usepackage{booktabs}
\usepackage{todonotes}

\usepackage{amssymb}
\usepackage{bm}
\usepackage{bussproofs}
\usepackage{multirow}
\usepackage{amsmath}
\usepackage[ruled,vlined,linesnumbered]{algorithm2e}
\usepackage[capitalize,nameinlink,noabbrev]{cleveref}
\usepackage{amsthm}
\usepackage{xspace}
\newcommand{\QXR}{\textsc{Qxr}\xspace}

\crefname{algocf}{algorithm}{algorithms}
\Crefname{algocf}{Algorithm}{Algorithms}
\definecolor{myblue}{RGB}{30,144,255}
\definecolor{other}{RGB}{90,90,255}

\crefalias{lemma}{lemma}
\crefname{lemma}{lemma}{lemmas}
\Crefname{lemma}{Lemma}{Lemmas}

\newtheorem{theorem}{Theorem}[section]
\newtheorem{corollary}[theorem]{Corollary}

\newtheorem{lemma}{Lemma}[section]
\theoremstyle{definition}
\newtheorem{definition}[theorem]{Definition}

\usepackage[T1]{fontenc}

\usepackage[utf8]{inputenc}

\usepackage{microtype}

\usepackage{inconsolata}

\usepackage{graphicx}
\usepackage{fontawesome5}
\usepackage{hyperref}

\title{Global Consistency Repair for Natural-Language Claim Sets}

\author{
  \textbf{Paul He\textsuperscript{1}}\thanks{Work begun while at Amazon Research, Tübingen.}
  \quad
  \textbf{Elke Kirschbaum\textsuperscript{2}}
  \quad
  \textbf{Shiva Prasad Kasiviswanathan\textsuperscript{2}}
  \\
  \textsuperscript{1}NTU Singapore
  \qquad
  \textsuperscript{2}Amazon
  \\
  \texttt{paul005@ntu.edu.sg}
  \qquad
  \texttt{elkeki@amazon.com}
  \\
  \texttt{kasivisw@gmail.com}
}

\begin{document}
\maketitle
\begin{abstract}
NLP systems increasingly aggregate claims from retrieved passages, documents, and generated reasoning steps. Yet current factuality methods mostly verify claims individually or compare pairs of claims, leaving open the question of whether the entire aggregated claim set can be retained coherently. This misses a
fundamental failure mode: claims may be locally plausible and
pairwise compatible while still being jointly inconsistent. We formalize global consistency repair: given an inconsistent natural-language claim set, preserve the largest coherent subset while localizing the compact sources of inconsistency. We introduce \QXR, an adaptive diagnosis-and-repair framework that treats an LLM as a noisy set-level consistency judge, localizes minimal conflicting subsets, and repairs the claim set by deleting a small hitting set of claims. Across LLM judges and benchmarks derived from VitaminC, FEVER, and conflicting-evidence RAG, \QXR preserves substantially more valid claims than direct all-at-once judging while maintaining high precision. These results suggest that factuality in claim aggregation
systems should be treated as a set-level repair problem, not only as
local verification.\footnote{Code available at \url{https://github.com/Hepaul7/QXR}}
\end{abstract}

\begin{figure*}
    \centering
    \includegraphics[width=1\linewidth]{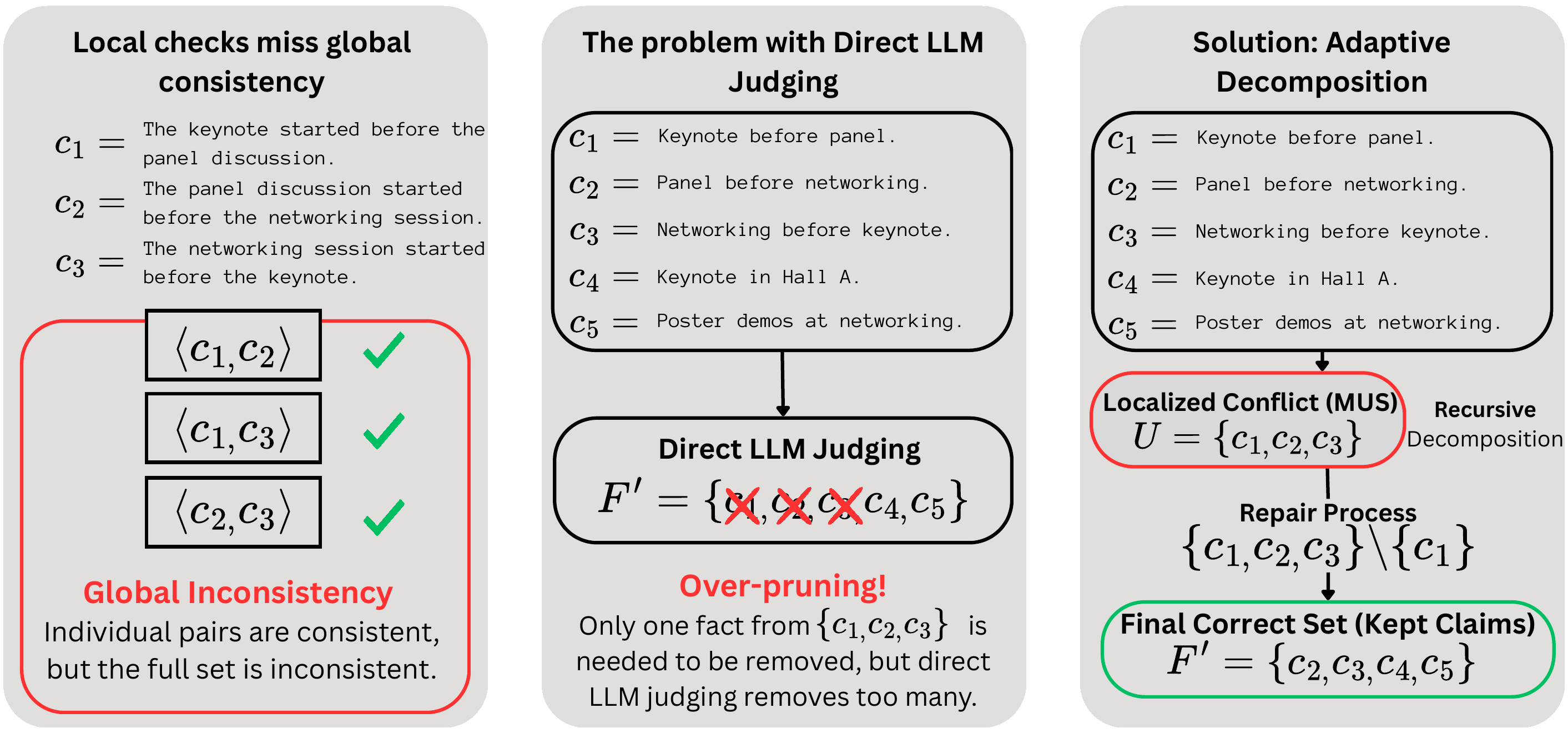}
    \caption{Why global consistency requires structured reasoning. Left: pairwise checks can all succeed even when the full claim set is globally inconsistent. Middle: direct all-at-once LLM judging may over-prune, removing more claims than necessary to restore consistency. Right: adaptive decomposition localizes a minimal unsatisfiable subset and repairs the set while preserving more valid claims.}
    \label{fig:mainfig}
\end{figure*}

\section{Introduction}
Natural-language systems increasingly assemble
\emph{sets of claims} from multiple documents,
retrieved passages, or extracted evidence
\cite{chen-etal-2024-complex, NEURIPS2023_cd86a305, aly-vlachos-2022-natural}.
Applications such as multi-document summarization,
fact checking, retrieval-augmented generation
(RAG) \cite{NEURIPS2020_6b493230}, and knowledge-base construction must
combine information originating from many
sources into a coherent final output. Yet
individually plausible claims need not be
\emph{jointly coherent}:
a RAG system may retrieve passages that are each
relevant to a query while still supporting
incompatible answers
\cite{wang2025retrievalaugmented}. The downstream challenge is therefore not only to
detect that a contradiction exists, but to
identify \emph{where} it arises while preserving
as much valid information as possible. A key difficulty in these settings is that
conflicts are often sparse:
only a small subset of claims may be mutually
incompatible, while most retrieved or generated
information remains useful.

To make this concrete, consider a retrieval-augmented system generating a compliance or incident report from dozens of retrieved passages. Individual statements may each be locally supported while still being jointly incompatible:
different documents may disagree on dates,
attribute conflicting causes to the same event,
or provide mutually inconsistent entity facts.
In such settings, discarding all retrieved
information is undesirable, while retaining every
locally plausible statement can yield incoherent
outputs.
The practical objective is therefore to preserve
the largest subset of claims that remains jointly
coherent while isolating only the conflicting
information. This paper asks a different question from standard factuality evaluation. Instead of asking whether a single claim is supported, or whether two claims contradict, we ask: given an inconsistent set of natural-language claims, which claims should an NLP system keep? This turns factuality for claim aggregation into a preservation-oriented repair problem.

We study this problem as \emph{global consistency repair}: given a collection of claims, retain the largest subset that remains jointly coherent. Unlike contradiction detection, this objective requires both identifying conflicts and repairing them while preserving unaffected information. To operationalize this goal, we draw inspiration
from inconsistency diagnosis:
Minimal Unsatisfiable Subsets (MUSes),
the smallest groups of claims that cannot jointly
be true, provide localized explanations of
inconsistency and naturally support targeted
repair.

Most NLP approaches to factuality and contradiction operate \emph{locally}:
they verify individual claims against evidence,
decompose complex claims into smaller units, or
detect contradictions between pairs of statements
\cite{wang-shu-2023-explainable,
tan-etal-2025-improving,
pan-etal-2023-fact,
de-marneffe-etal-2008-finding}.
These methods are useful, but they do not solve
the set-level problem studied here.
A collection of claims may pass local verification
and remain pairwise compatible while still being
jointly inconsistent.
Consequently, local factuality pipelines cannot in
general certify that an aggregated evidence set or
final system output is globally coherent.
This gap becomes increasingly important in
long-context RAG, multi-document information
extraction, and report generation, where outputs
depend on reasoning over an entire claim set
rather than isolated statements.

Large language models (LLMs) provide a practical interface for set-level consistency judging: modern models can often assess whether a small subset of claims is mutually compatible \cite{li2024llmsasjudgescomprehensivesurveyllmbased}, though LLMs are known to be logically inconsistent on complex queries \cite{ghosh2025logical}. Furthermore, exhaustive subset checking is infeasible, pairwise checks cannot certify global coherence, and direct ``all-at-once'' prompting becomes brittle as claim sets grow. We therefore formalize \emph{global consistency repair} as a distinct NLP capability and introduce \QXR (Query-eXplain-and-Repair), an adaptive diagnosis-and-repair framework that uses LLM consistency judgments to localize compact conflicting subsets and remove only the claims needed to restore coherence. We provide theoretical motivation for structured decomposition and show empirically, across factuality benchmarks and realistic RAG settings, that \QXR preserves substantially more valid claims than direct all-at-once judging while maintaining high precision. \Cref{fig:mainfig} illustrates the core failure mode and repair strategy.

\section{Related Work}

\paragraph{Local Factuality and Claim Verification.}
Fact verification and factual consistency checking
with LLMs have attracted growing attention,
motivated by challenges such as hallucinations,
misinformation, and unreliable generation
\cite{rahman2025hallucinationtruthreviewfactchecking,
singhal-etal-2024-multilingual, guo-etal-2022-survey}.
One line of work focused on extracting structured
knowledge units, such as subject--predicate--object
triples, from model outputs and reference texts to
identify local inaccuracies and support
retrieval-augmented verification
\cite{chen-etal-2025-graphcheck,
10.1609/aaai.v39i22.34520}.
Other methods decompose responses into atomic
claims that are verified independently
\cite{hu-etal-2025-decomposition,
wanner-etal-2024-closer},
which works well for short responses but becomes
increasingly brittle in long-form or
multi-document settings where interactions among
claims matter.
Reranking and retrieval-based methods
\cite{liu-etal-2025-pointwise}
can partially mitigate these issues, but often
depend on model internals or scoring signals that
are unavailable in black-box LLM settings.

\paragraph{Set-Level Consistency in NLP Systems.}
Beyond local factuality, substantial work has studied contradiction detection and coherence, often through natural-language inference (NLI)
formulations
\cite{thorne-etal-2018-fever,
kryscinski-etal-2020-evaluating}. Benchmarks such as FEVER and VitaminC focus on detecting whether individual claims are supported or contradicted by evidence \cite{thorne-etal-2018-fever,
schuster-etal-2021-get}, but do not address whether an \emph{entire set} of claims can jointly remain coherent.  More recently, LLMs have been used directly as
judges for factuality and coherence
\cite{zheng2023judging}, and earlier work has also studied logical
consistency of LLMs on propositional or knowledge-graph reasoning tasks \cite{ghosh2025logical}. Prior work that repairs set-level inconsistency either materializes constraints before solving -- e.g., hand-authored rules \cite{kassner-etal-2021-beliefbank}, pairwise-NLI \cite{mitchell-etal-2022-enhancing}, generated explanations \cite{jung-etal-2022-maieutic} -- or learns a supervised set-level scorer \cite{song-etal-2025-introducing}. Our work takes neither route: it treats consistency judging as a black-box oracle and localizes conflicts through adaptive subset queries, which is what allows it to surface conflicts that no pairwise constraint set can represent, without task-specific training.

\paragraph{Positioning of Our Work.}
To our knowledge, this is the first work to formalize \textit{global consistency repair} as a constrained optimization over natural-language claim sets: maximizing retained claims subject to scope-wise consistency, characterized through hitting-set duality and query complexity. This preservation-oriented objective is central to RAG, multi-document summarization, and evidence aggregation, where outputs depend on the joint coherence of many interacting claims. Methodologically, \QXR combines inconsistency diagnosis with LLM-based consistency judgments, localizing compact conflict sources and removing only the claims needed to restore coherence. Although the components we used are established tools for diagnosing over-constrained systems, what changes here is the setting: the constraints are natural-language claims with no symbolic encoding, and the consistency test is a noisy black-box judge, hence oracle calls are the dominant cost rather than solver time. We now formalize this objective as the preservation-oriented counterpart to local contradiction detection.

\section{Problem: Global Consistency Repair}
\paragraph{Problem Definition.}
Let $F$ be a finite set of \textbf{facts} and
$C=\{C_1,\dots,C_m\}$ be a \textbf{family of scopes}
with $C_i \subseteq F$, where each scope groups
facts whose joint consistency matters.
For example, a scope may correspond to a retrieved
evidence set in RAG, a cluster of claims in
multi-document summarization, or an entity-specific
subset of facts in knowledge aggregation.

We assume an underlying \textbf{consistency function} $A:2^F \rightarrow
\{\mathsf{cons},\mathsf{incons}\}$ which determines whether a subset of facts is
jointly coherent.
Our objective is to \textbf{retain} a subset
$F' \subseteq F$ that maximizes coverage
(i.e., preserves as much information as possible)
while remaining consistent within every scope:
\begin{align*}
\max_{F'\subseteq F}
\ |F'|
\quad \text{s.t.}\quad
A(F' \cap C_i)=\mathsf{cons},
\ \forall i\in[m].
\end{align*}

That is, the objective is explicitly
\emph{preservation-oriented}:
rather than rejecting an entire inconsistent
claim set, we seek the \emph{largest subset of
claims that remains jointly coherent} under all
relevant scopes.
This distinction is important in NLP systems,
where downstream outputs often depend on broad
evidence coverage.
Overly aggressive filtering may discard many
correct statements, while insufficient filtering
leaves contradictions unresolved. Equivalently, the problem can be viewed as
removing the fewest claims necessary to restore
coherence.
The corresponding minimum-deletion formulation is
given in \Cref{appendix:objective}. We also show in \Cref{appendix:hitting_reduction} that the minimum-deletion view is equivalent to a hitting-set problem over the family of all MUSes.

\paragraph{Definitions.}
We briefly introduce terminology used throughout
the paper.
A subset $U \subseteq F$ is a
\textbf{minimal unsatisfiable subset (MUS)}
if $A(U)=\mathsf{incons}$, but every proper subset
$U' \subset U$ satisfies
$A(U')=\mathsf{cons}$.
Intuitively, a MUS is a subset-minimal localized 
group of claims responsible for a consistency
failure.
These subsets provide compact explanations of
\emph{where} inconsistency arises.

A \textbf{repair} is a deletion set
$R \subseteq F$ such that the retained facts
$F \setminus R$ remain consistent within all
relevant scopes.
If $\mathcal U$ is a family of discovered MUSes,
then a set $H \subseteq F$ is a
\textbf{hitting set} if
$H \cap U \neq \varnothing$
for every $U \in \mathcal U$.
In our setting, removing a hitting set guarantees
that at least one claim is removed from each
identified conflict, thereby restoring consistency.
Thus, MUSes provide localized explanations of
inconsistency, while hitting sets provide the
corresponding repair mechanism.

\paragraph{LLM-based Consistency Judging.}
In practice, the consistency function
$A$
is not directly observable, and constructing
domain-specific symbolic verifiers for open-ended
natural-language claims is often impractical.
We therefore approximate consistency judgments
using a pretrained LLM.
For any subset of claims
$S \subseteq F$,
we prompt the model to determine whether the
claims in $S$ are mutually coherent, yielding a
consistency judgment $O(S)\in
\{\mathsf{cons},\mathsf{incons}\}$.

LLMs provide a practical interface for set-level
consistency checking because they can often
reliably assess whether \emph{small} groups of
claims appear mutually compatible.
At the same time, these judgments are imperfect:
responses may be stochastic, sensitive to prompt
formatting, or degrade as the number of
interacting claims grows.
We therefore view
$O$
as a noisy approximation to the underlying notion
of global coherence.
For theoretical analysis, we treat
$A$
as an idealized monotone consistency function and
consider assumptions such as perfect judging only
to analyze the structure of the problem and
provide correctness guarantees.
The empirical method itself does not rely on
these assumptions.

\subsection{Problem Hardness}
\label{sec:hardness}
Global consistency repair presents two challenges. Exact repair is computationally hard, and the relevant inconsistencies may depend on interactions that cannot be detected locally. We establish both points before motivating \QXR’s adaptive decomposition.

\paragraph{Worst-Case Hardness.}
In the worst case, solving the maximum-preservation objective is NP-hard. The ground-truth consistency function $A$ can encode arbitrary satisfiability structure, so exact repair can require solving a general constraint-satisfaction problem. Equivalently, let $\mathcal U$ denote the family of all minimal unsatisfiable subsets of $F$. A deletion $R \subseteq F$ is feasible exactly when it intersects every $U \in \mathcal U$. Thus, exact repair can be viewed as finding the
smallest set of claims whose removal breaks all
conflicts, that is, a minimum hitting set over
the discovered inconsistencies.

\paragraph{Pairwise Consistency is Not Enough.} Pairwise consistency does not imply global consistency. For any $N\geq 3$, there exists a fact set $F$ of size $\binom{N}{2}$ such that every pair of facts in $F$ is jointly consistent, yet $F$ itself is globally inconsistent. One way to view the construction is through a graph-coloring abstraction: pairwise incompatibility constraints can all be satisfiable in isolation, while the full set requires more global structure than the available assignments permit. This formalizes a core limitation of current fact-checking and factuality pipelines: methods that verify claims individually, or only compare claim pairs, cannot in general certify that a full set of claims is jointly coherent. Full proof details are given in \cref{appendix:pairwise_insuff}.

\paragraph{Worst-Case Global Interactions.}
The obstruction can also be genuinely global. Consider Boolean variables $z_1,\ldots,z_{N+1}$ and fixed bits $x_1,\ldots,x_N\in\{0,1\}$, with claims of the form $z_{i+1}=z_i\oplus x_i$ for each $i\in[N]$, together with the boundary claim $z_{N+1}=z_1$. Composing the constraints shows that the set is satisfiable iff $\bigoplus_{i=1}^N x_i=0$. Thus, flipping any single $x_i$ changes the global consistency label. In the worst case, consistency can therefore behave like parity: every claim may be pivotal, and the inconsistency may not be attributable to any small local check. This illustrates why direct all-at-once judging can require aggregating long-range interactions across the entire claim set. This observation is consistent with analyses in \citet{hahn-rofin-2024-sensitive} showing that parity-like functions are difficult for transformers to model robustly, although those results should not be interpreted as a direct lower bound for decoder-only LLM judges.

\paragraph{Na\"ive Decomposition Fails; Adaptive Helps.}
\label{para:decomposition}
The analysis so far suggests that global
consistency cannot, in general, be recovered by
simply decomposing a claim set into smaller parts
and aggregating the results independently.
Suppose an inconsistent claim set
$S \subseteq F$
is partitioned as
$S=S_1\cup S_2$.
Then at least one of the following must hold:
(i) $S_1$ is inconsistent, in which case some
MUS is fully contained within $S_1$;
(ii) $S_2$ is inconsistent, so some MUS is fully
contained within $S_2$; or
(iii) both $S_1$ and $S_2$ are individually
consistent, yet their union
$S_1\cup S_2$
is inconsistent, implying that the relevant MUS
crosses the partition and depends on claims from
both subsets. 
This distinguishes our setting from decomposable query problems, where answers on disjoint subsets can be combined to obtain the answer on their union, enabling generic reductions to efficient algorithms. In global consistency repair, by contrast, this is not possible due to case (iii). The one-MUS case is the simplest counterexample showing that naïve decomposition can fail. With multiple MUSes, each MUS may lie within one partition or cross the partition, and a single crossing MUS is enough to invalidate independent decomposition.

Naïve divide-and-conquer strategies or local verification pipelines may therefore miss contradictions that emerge only across groups of claims. At the same time, the existence of localized
conflicts suggests that \emph{structured}
decomposition can still be useful:
the system should recursively localize compact
sources of inconsistency while preserving
globally coherent information.
This motivates the adaptive diagnosis-and-repair
framework introduced in the next section.

\section{Conflict Diagnosis and Repair}
\label{sec:method}

We now describe \QXR, our framework for global consistency repair using LLM-based consistency judgments. As motivated by the decomposition trichotomy above, \QXR uses each split not to aggregate independent answers, but to decide whether to recurse into one side or preserve cross-partition context during localization. Formally, \QXR takes as input a fact set $F$, a family of scopes $ C=\{C_1,\ldots,C_m\}$ with $C_i\subseteq F$, and an LLM-based consistency judge $O$. It returns a repaired subset $F' \subseteq F$ and, optionally, the localized conflicts $\mathcal U_{\mathrm {found}}$ discovered during repair. \QXR alternates between localizing MUSes and removing a small hitting set until all relevant scopes are judged consistent. By localizing these conflict regions, \QXR avoids the over-pruning behavior often observed in direct all-at-once judging. Formal guarantees for soundness, approximation, and query complexity are given in \Cref{app:proof}.

\subsection{Diagnosis-and-Repair Overview}

\Cref{alg:main} gives the overall \QXR procedure. Starting from $F' \gets F$, \QXR repeatedly identifies the scopes currently judged inconsistent,
\begin{align*}
\mathcal{T}_{\mathsf{incons}}
\triangleq
\{\, C_j \in C : O(C_j)=\mathsf{incons}\,\},
\end{align*}
localizes one MUS $U_j \subseteq C_j$ for each such scope, and collects the resulting conflict family
\begin{align*}
\mathcal{U} \triangleq \{\, U_j : C_j \in \mathcal{T}_{\mathsf{incons}} \,\}.
\end{align*}
These discovered conflicts are also accumulated into $\mathcal{U}_{\mathrm{found}}$, which may optionally be returned as a record of surfaced inconsistencies.

\QXR then computes a repair set $H$ that intersects every conflict: $\forall U \in \mathcal{U}, H \cap U \neq \varnothing$. It removes $H$ from the retained fact set and restricts all scopes accordingly. The process repeats until no scope remains inconsistent. Thus, \QXR alternates between diagnosis, which exposes compact conflict witnesses, and repair, which removes only facts needed to break those conflicts.

\subsection{Conflict Localization}
\QXR assumes access to a MUS-localization routine \textsc{LocalizeMUS}$(O,S,B)$, where $S\subseteq F$ is the candidate set being minimized and $B$ is a background context assumed to satisfy $O(B)=\mathrm{cons}$. The goal is to return a subset $U\subseteq S$ such that $B\cup U$ is inconsistent, but removing any element of $U$ restores consistency relative to $B$. Thus, $U$ is a subset-minimal explanation of inconsistency conditioned on the current background.

The background context is essential for handling cross-partition conflicts. If an inconsistent candidate $S=S_1\cup S_2$ satisfies $O(B\cup S_1)=\mathrm{incons}$, then localization can recurse into $S_1$ while keeping $B$ fixed; the case for $S_2$ is analogous. The harder case is when
$O(B\cup S_1)=O(B\cup S_2)=\mathsf{cons}$
but $O(B\cup S_1\cup S_2)=\mathsf{incons}$.
Here, no conflict is contained entirely within either side relative to $B$,
so the inconsistency requires claims from both $S_1$ and $S_2$. Hence, the localizer minimizes one side while treating the other as background, then minimizes the remaining side relative to the reduced context. This conditional use of background distinguishes MUS localization from na\"ive divide-and-conquer.

We use QuickXplain \citep{1597148.1597177} as the conflict-localization routine. The algorithm uses divide-and-conquer to identify subset-minimal inconsistencies. QuickXplain serves as a replaceable localization primitive; \QXR is the surrounding NLP repair framework that combines noisy LLM consistency judgments, conflict localization, and preservation-oriented hitting-set repair. Under a perfect consistency judge, QuickXplain localizes a
conflict of size $k$ within a candidate set
$S$
using
$\mathcal{O}(k\log |S|)$
consistency queries.
Its recursive conditioning naturally aligns with
the cross-partition conflict structure described
above.
We defer pseudocode to
\Cref{appendix:qx}.

\begin{algorithm}[t]
\caption{\QXR: Global Consistency Diagnosis and Repair}
\label{alg:main}
\KwIn{Facts $F$; scopes $C=\{C_1,\dots,C_m\}$; noisy oracle $O$}
\KwOut{Repaired fact set $F'$; optionally discovered conflicts $\mathcal{U}_{\mathrm{found}}$}

$F' \gets F$\;
$\mathcal{U}_{\mathrm{found}} \gets \varnothing$\;

\While{$\exists\, C_j \in C \text{ s.t. } O(C_j)=\mathsf{incons}$}{
    $\mathcal{U} \gets \varnothing$\;
    
    $\mathcal{T}_{\mathsf{incons}} \gets \{\,C_j \in C : O(C_j)=\mathsf{incons}\,\}$\;
    
    \ForEach{$C_j \in \mathcal{T}_{\mathsf{incons}}$}{
        $U_j \gets \textsc{LocalizeMUS}(O, C_j, \varnothing)$\;
        $\mathcal{U} \gets \mathcal{U} \cup \{U_j\}$\;
    }
    
    $\mathcal{U}_{\mathrm{found}} \gets \mathcal{U}_{\mathrm{found}} \cup \mathcal{U}$\;
    
    $H \gets \textsc{GreedyHittingSet}(\mathcal{U})$\;
    
    $F' \gets F' \setminus H$\;
    $C \gets \{\,C_j \setminus H : C_j \in C\,\}$\;
}
\Return{$F'$ and optionally $\mathcal{U}_{\mathrm{found}}$}
\end{algorithm}

\subsection{Minimal Repair via Hitting Sets}

Given localized conflicts $\mathcal{U}$, any valid repair must remove
at least one claim from each conflict. \QXR therefore seeks a hitting set
$H \subseteq F$ such that $H \cap U \neq \varnothing$ for every
$U \in \mathcal{U}$, and returns the repaired claim set
$F' = F \setminus H$. This separates diagnosis from repair:
localization identifies compact explanations of inconsistency, while
the hitting set selects a small deletion set that breaks all discovered
conflicts and preserves the remaining claims.

We instantiate this step with a greedy hitting-set solver, which is
efficient and has the standard logarithmic approximation guarantee
(\Cref{lem:hittingset}). When few conflicts are discovered, the same optimization can instead be solved exactly by enumeration or integer programming. Subsequent consistency checks detect conflicts not yet localized.

\subsection{Query Complexity}

Because LLM consistency judgments are the dominant
computational cost, an important question is how
many oracle calls \QXR requires in practice.
The following result characterizes the complexity
of the diagnosis-and-repair loop.

\begin{theorem}[Query Complexity of
\Cref{alg:main}]
Let $N=|F|$ be the number of claims, $m=|C|$
the number of scopes, $k$ the maximum size of any MUS localized by \QXR, and $I$
the number of outer iterations until termination.
Assuming \textsc{LocalizeMUS} is instantiated with QuickXplain, \Cref{alg:main}
makes at most $\mathcal{O}\!\left(
I \cdot m \cdot
(1+k\log N)
\right)$ consistency queries.
For $k\ge 1$,
this simplifies to $\mathcal{O}\!\left(
I\cdot m\cdot k\log N
\right).$
\end{theorem}

A proof is given in \Cref{appendix:query_complexity}. The dependence on $k$ is important: query complexity scales with the size of the
localized conflicts rather than the full claim set.

Although the number of scopes $m$ can be exponential in the worst case, many NLP settings involve only a small number of natural scopes. An important observation is that \QXR does not enumerate the full family of MUSes, which can itself be exponential in size. Instead, \QXR adaptively localizes conflicts only in currently inconsistent scopes, repairs the claim set, and continues diagnosis only if inconsistency remains as judged by the oracle. Therefore, the number of oracle calls depends on the conflicts encountered during repair and their localized size $k$, rather than on enumerating all possible conflicts. When conflicts are sparse and localized ($k\ll N$), this can further reduce the cost of conflict localization.

To illustrate the above point, consider the fully global setting $\mathcal C=\{F\}$. If a scope contains another scope, consistency of the larger scope already implies consistency of the smaller one, making the nested scope redundant. We formalize this as \emph{scope subsumption} in \Cref{app:structural_theory}. Hence, when the entire claim set $F$ is the scope of interest, we can take $m=1$, reducing the bound to $\mathcal O(I\cdot k\log N)$. In this case, the explicit dependence on the claim-set size $N$ is only logarithmic. This does not imply logarithmic scaling overall, as $k$ and $I$ can still grow with $N$ in the worst case. Although pairwise comparison—requiring $\binom{N}{2}\in\mathcal O(N^2)$ checks—provably cannot solve the set-level problem, it is useful to illustrate how quickly the number of oracle calls can scale with $N$: $435$ calls at $N=30$ and $4{,}950$ at $N=100$.

Additional structure can also make exact global consistency repair more tractable when $m>1$. Fact-disjoint scope components can be repaired independently. More generally, consider the conflict graph induced by the family of scope-contained MUSes, where two facts are adjacent if they occur together in some MUS. If this conflict graph has bounded treewidth $t$, exact repair can be solved in $2^{t+1}\cdot \mathrm{poly}(|F|+|\mathcal A_{\mathrm{incons}}|)$ time, given $\mathcal A_{\mathrm{incons}}$ and a tree decomposition of width $t$. Intuitively, low conflict-graph treewidth means that inconsistencies involve localized interactions among facts, rather than many facts being densely interconnected through overlapping conflicts. Thus, the exponential dependence is on the treewidth of the conflict graph rather than the total number of claims (\Cref{app:structural_theory}).

\begin{table}[t]
\centering
\begin{tabular}{lcc}
\toprule
& VitaminC & FEVER \\
\midrule
MUSes / inst. & 4.05 [2, 6] & 4.77 [0, 28] \\
Min. HS size  & 4.05 [2, 6] & 4.01 [0, 8] \\
Facts / inst. & 28.11 [24, 32] & 33.73 [20, 62] \\
\bottomrule
\end{tabular}
\caption{\textbf{Conflict-structure statistics.}
Values report mean [min, max] across instances.}
\label{tab:conflict_stats}
\end{table}

\begin{table*}[h!]
\centering
\setlength{\tabcolsep}{3pt}
\begin{tabular}{l|ccc|ccc|ccc|ccc}
\toprule
& \multicolumn{6}{c|}{\textbf{VitaminC}} & \multicolumn{6}{c}{\textbf{FEVER}} \\
\cmidrule(lr){2-7} \cmidrule(lr){8-13}
& \multicolumn{3}{c}{\textbf{Direct (baseline)}} & \multicolumn{3}{c|}{\textbf{\QXR (ours)}}
& \multicolumn{3}{c}{\textbf{Direct (baseline)}} & \multicolumn{3}{c}{\textbf{\QXR (ours)}} \\
\cmidrule(lr){2-4} \cmidrule(lr){5-7} \cmidrule(lr){8-10} \cmidrule(lr){11-13}
\textbf{Model}
& \textbf{P} & \textbf{R} & \textbf{F1}
& \textbf{P} & \textbf{R} & \textbf{F1}
& \textbf{P} & \textbf{R} & \textbf{F1}
& \textbf{P} & \textbf{R} & \textbf{F1} \\
\midrule
Sonnet 3.7
& \textbf{0.979} & 0.854 & 0.912
& 0.956 & \textbf{0.975} & \textbf{0.965}
& \textbf{0.992} & 0.805 & 0.889
& 0.983 & \textbf{0.977} & \textbf{0.980} \\
Sonnet 4
& \textbf{0.956} & 0.877 & 0.915
& 0.938 & \textbf{0.983} & \textbf{0.960}
& \textbf{0.995} & 0.833 & 0.907
& 0.981 & \textbf{0.977} & \textbf{0.978} \\
Sonnet 4.5
& \textbf{0.997} & 0.617 & 0.762
& 0.975 & \textbf{0.957} & \textbf{0.966}
& \textbf{0.990} & 0.901 & 0.944
& 0.978 & \textbf{0.947} & \textbf{0.962} \\
Opus 4.6
& 0.654 & 0.574 & 0.611
& \textbf{0.985} & \textbf{0.968} & \textbf{0.976}
& \textbf{0.995} & \textbf{0.984} & \textbf{0.989}
& 0.971 & 0.958 & 0.964 \\
DeepSeek-R1
& \textbf{0.980} & 0.730 & 0.837
& 0.973 & \textbf{0.990} & \textbf{0.981}
& \textbf{0.989} & 0.821 & 0.897
& 0.988 & \textbf{0.980} & \textbf{0.983} \\
GPT-OSS-120B
& \textbf{0.984} & 0.926 & 0.954
& 0.956 & \textbf{0.995} & \textbf{0.975}
& 0.976 & 0.975 & 0.976
& \textbf{0.992} & \textbf{0.980} & \textbf{0.985} \\
Mistral Large 2
& 0.955 & 0.603 & 0.739
& \textbf{0.968} & \textbf{0.978} & \textbf{0.972}
& \textbf{0.970} & 0.780 & 0.865
& 0.964 & \textbf{0.990} & \textbf{0.976} \\
\bottomrule
\end{tabular}
\caption{
\textbf{Evaluation of consistent fact sets $F'$ on two datasets (VitaminC \cite{schuster-etal-2021-get} / FEVER \cite{thorne-etal-2018-fever}).}
Precision (P), recall (R), and F1 are computed with respect to gold consistent subsets.
\QXR yields cleaner and more complete $F'$ than direct all-at-once LLM judging.
}
\label{tab:qx_results_compact}
\end{table*}

\section{Experiments}

\subsection{Experimental Setup}

\paragraph{Datasets.}
We construct global-consistency instances from
VitaminC~\citep{schuster-etal-2021-get}
and
FEVER~\citep{thorne-etal-2018-fever}
by grouping approximately 30 factual claims per
example.
Because benchmark resources for \emph{set-level}
consistency verification remain limited, we adapt
existing fact-checking datasets into controlled
testbeds for global consistency repair.
This is a natural setting:
both datasets contain semantically related claims
with supporting and contradicting evidence,
allowing us to study scenarios where individual
claims may appear locally plausible while the
overall claim set remains jointly inconsistent.
The resulting setup enables controlled evaluation
of both consistency restoration and
\emph{preservation of valid claims}.

In VitaminC, each cluster contains
24--32 claims in total, including 2--6 injected contradictions derived from
\texttt{REFUTES} edits.
Each contradiction forms a size-2 ground-truth
MUS
$\{c^+,c^-\}$.
In FEVER, we combine
0--8
\texttt{REFUTES} claims (with evidence) with
14--18
\texttt{SUPPORTS} claims.
Ground-truth MUSes pair refuting claims with each evidence set they contradict; removing either the refuting claim or its evidence restores consistency, producing dense
multi-claim interactions that require set-level
reasoning. \Cref{tab:conflict_stats} summarizes the resulting conflict structure. VitaminC instances have a relatively uniform number of disjoint conflicts, whereas FEVER is more heterogeneous and can contain multiple MUSes sharing claims, resulting in fewer required deletions than MUSes in some instances.

\paragraph{Models.}
We evaluate both direct all-at-once judging and
\QXR
using a diverse set of LLM consistency judges:
Claude Sonnet 3.7,
Claude Sonnet 4,
Claude Sonnet 4.5,
Claude Opus 4.6,
DeepSeek-R1,
GPT-OSS-120B,
and Mistral Large (2407)
\citep{anthropic2025claude37,
anthropic2025claude4,
anthropic2025claude45,
anthropic2025claude46,
Guo_2025, openai2025gptoss,
mistrallarge}.
This selection spans both state-of-the-art closed models and strong open-weight models, thus allowing us to test whether the global-consistency gap persists across diverse model families. All experiments use Amazon Bedrock with temperature $0.0$
and maximum output length 8192
tokens. We use deterministic decoding to reduce variance across runs and apply the same zero-shot consistency prompt across methods unless otherwise stated.

\paragraph{Direct Baseline.}
We compare against a
\emph{direct all-at-once baseline}
that prompts the LLM once over the entire claim
set
$F$
and asks it to return the largest jointly
coherent subset
$F' \subseteq F$.
This baseline reflects the most natural practical
alternative to \QXR:
rather than localizing conflicts explicitly, the
model attempts to resolve global inconsistencies
through a single holistic judgment.
Our central question is whether structured
diagnosis improves the
\emph{preservation of valid claims}
relative to this direct approach.
We discuss why exhaustive pairwise or
NLI-style baselines are not computationally
comparable at scale in
\Cref{app:nli_baselines}.

\paragraph{\QXR Setting.}
Unless otherwise stated, all \QXR experiments use
a single oracle call per query
($r=1$) and the scope setting
$C=\{F\}$. This is the fully global setting: the entire claim set must be judged jointly consistent, and under monotonicity all nested subscopes are redundant (\Cref{app:scope_sub_proof}). We do not use majority voting or repeated
querying in the empirical evaluation;
assumptions about repeated-query independence are
introduced only in the theoretical analysis.
This isolates the contribution of structured
conflict localization and repair without relying
on additional sampling or self-consistency
mechanisms.
Prompt templates for both the subset-consistency
oracle and the direct baseline are provided in
\Cref{app:prompts}.

\paragraph{Evaluation.}
Given an initial claim set
$F$,
the model returns a repaired subset
$F'$
after removing claims identified as inconsistent.
We evaluate
$F'$
against the
\textbf{gold coherent subset}
$F_{\text{gold}}$,
defined as the maximal subset of
$F$
containing no injected contradictions
(i.e., all ground-truth satisfiable claims).
Precision, recall, and F1 are computed over the
retained claims:
\[
P=
\frac{|F' \cap F_{\text{gold}}|}
{|F'|},
\qquad
R=
\frac{|F' \cap F_{\text{gold}}|}
{|F_{\text{gold}}|}.
\]
We identify the set of valid minimum repairs and evaluate a prediction against the best-matching valid gold repair, therefore equivalent optimal repairs are not penalized. \QXR’s greedy hitting-set solver may also admit tied choices, which we resolve deterministically. This evaluation directly measures the preservation objective:
high recall reflects retaining valid claims,
while high precision ensures inconsistent claims are removed, matching the formal objective explained in \Cref{lem:equiv_obj}.

\subsection{Results and Analysis}

\paragraph{Direct Judging Tends to Over-prune.}
\Cref{tab:qx_results_compact} shows the results of experiments for VitaminC and FEVER. The dominant failure mode of direct all-at-once judging is not low precision but low recall: once an inconsistency is detected, the model often removes substantially more claims than are necessary to restore coherence. We refer to this behavior as \emph{over-pruning}. This failure mode is not eliminated by stronger prompting: in \Cref{app:prompt}, direct judging under chain-of-thought, decomposition, few-shot, and self-consistency prompts still tends to preserve high precision but lower recall, indicating that the issue is the lack of explicit conflict localization rather than prompt sensitivity. Because \QXR first localizes compact conflicts and then removes only a hitting set of inconsistent claims, it directly targets this failure mode. As a result, \QXR typically preserves substantially more valid claims while maintaining competitive precision.

The gains are especially pronounced on VitaminC. For example, Claude Sonnet 4.5 improves from $R=0.617$ to $R=0.957$, while Claude Opus 4.6 improves from $R=0.574$ to $R=0.968$. One possible reason is VitaminC’s contrastive construction, as it is built from Wikipedia revisions where supporting and refuting evidence differ by small factual edits \cite{schuster-etal-2021-get}. Therefore, contradictions are often subtler and depend on fine-grained factual details, which may make direct all-at-once repair harder even on frontier models. On FEVER, conflicts are often simpler; hence, direct prompting is already strong for some models. Nevertheless, \QXR remains competitive on both benchmarks and yields meaningful recall improvements. Hence, \QXR’s gains are primarily driven by \emph{preserving} more coherent claims rather than more aggressive filtering.

\begin{table}[t]
\centering
\setlength{\tabcolsep}{3pt}
\begin{tabular}{lcccccc}
\toprule
\multirow{2}{*}{Model} 
& \multicolumn{3}{c}{Direct} 
& \multicolumn{3}{c}{\QXR} \\
\cmidrule(lr){2-4} \cmidrule(lr){5-7}
& P & R & F1 & P & R & F1 \\
\midrule
Qwen2.5 7B & .883 & .639 & .741 & \textbf{.902} & \textbf{.743} & \textbf{.815} \\
Qwen2.5 32B & .852 & .839 & .845 & \textbf{.854} & \textbf{.998} & \textbf{.920} \\
Haiku 4.5 & \textbf{.979} & .575 & .724 & .925 & \textbf{.810} & \textbf{.863} \\
Sonnet 4.5 &\textbf{.962}& .814 & \textbf{.882} & .887 & \textbf{.854} &    {.870}\\

DS-v4-Pro & \textbf{.905} & .822 & .861 & .860 & \textbf{.952} & \textbf{.904} \\
\bottomrule
\end{tabular}
\caption{RAMDocs consistency repair results.}
\label{tab:ramdocs-main}
\end{table}

\subsection{Application: RAG}
We now consider a more realistic claim aggregation setting. To do this, we use RAMDocs~\cite{wang2025retrievalaugmented}, a RAG benchmark containing queries with conflicting retrieved evidence. We evaluate \QXR using Qwen2.5 \cite{qwen2025qwen25technicalreport}, DeepSeek-v4-Pro (DS-v4-Pro) \cite{deepseekai2026deepseekv4highlyefficientmilliontoken}, Claude Sonnet 4.5 \cite{anthropic2025claude45} and Claude Haiku 4.5 \cite{anthropic2025haiku45}. 
Each example consists of a query
$q$ and retrieved snippets $\mathcal D=
\{ d_1,\ldots, d_n\}$. 
We convert retrieved snippets into a claim set $F=\{f_1,\ldots,f_N\}$ via deterministic chunking and retain examples with annotated conflicts. Gold conflicts are represented by $\mathcal U_{\mathrm{gold}}$, where each
$U \in \mathcal U_{\mathrm{gold}}$ contains snippets that provide incompatible answers or mutually inconsistent evidence for the same query $q$. For our experiments $N=30$ with $|\mathcal U_\mathrm{gold}|=10$.

\Cref{tab:ramdocs-main} shows a similar trend to VitaminC and FEVER, where iterative conflict localization reduces the tendency observed in direct judging to \emph{over-prune} retrieved evidence snippets, thus allowing \QXR to preserve more useful information while removing contradictory claims. Furthermore, the results on the 7B and 32B models suggest that structured conflict localization may help compensate for weaker direct set-level reasoning. The results on RAMDocs further suggest that structured diagnosis-and-repair can provide practical gains in realistic RAG settings beyond controlled fact-checking benchmarks.

\section{Discussion}
\subsection{Experimental Observations}

\paragraph{When Does Structured Repair Help?} Our results suggest that \QXR is most useful when a claim set is largely coherent but contains localized conflicts, making unnecessary deletion costly. In such settings, direct judging may correctly detect inconsistency while still producing a poor repair. The smaller gains on FEVER also suggest that structured repair need not be preferable when direct judging is already reliable. \QXR is therefore best viewed as a preservation-oriented alternative when retaining unaffected information justifies additional oracle calls.

\paragraph{Detection $\not\Rightarrow$ Precise Localization.} Our results also highlight a distinction between \emph{detecting} inconsistency and \emph{localizing} its source. Direct judging often maintains high precision but has lower recall, suggesting that a model may recognize that some claims need to be removed while still failing to identify the conflicting claims precisely, resulting in \emph{over-pruning}. The improvements from \QXR therefore suggest that consistency detection alone does not provide a complete picture of a model’s ability to reason over conflicting claim sets. This also appears to persist under the stronger prompting strategies in \Cref{app:prompt}. Hence, conflict detection and localization should be considered separately.

\paragraph{Global Consistency as an NLP Objective.}
Taken together, these results suggest that
global consistency is not fully captured by
existing local verification pipelines.
In the contexts of RAG, multi-document summarization, and report generation, the difficulty is not limited to merely recognizing whether there is a contradiction; rather, it includes the challenge of re-establishing coherence without losing valid information. \QXR is not only useful in its ability to detect contradictions but also in its framing of global consistency as a \emph{diagnosis and repair} task on sets of natural language claims.

\subsection{\QXR Extensions}
\label{sec:qxr_ext}
\paragraph{Weighted Claims and Repair.}
The minimum-deletion objective studied in this work assigns equal cost to removing every claim. In practice, however, some claims might have lower confidence, come from less reliable sources, be outdated, or be less important. A natural generalization assigns each claim $f\in F$ a non-negative deletion cost $w(f)\geq0$, and replaces the minimum-cardinality hitting set with a minimum-weight hitting set:
\begin{align*}
\min_{H\subseteq F} \sum_{f \in H} w(f) \quad \text{s.t. } H\cap U \neq \varnothing, \ \forall U \in \mathcal U.
\end{align*}
The special case where $w(f)=1$ for all $f\in F$ is the objective considered in this work. This weighted formulation allows \QXR to prefer removing lower-confidence or less reliable claims when multiple repairs can restore consistency. Furthermore, this is important when conflicts overlap, where different minimum-cardinality repairs may remove the same number of claims but preserve information of different value.

\paragraph{Graded Oracle.} Real-world consistency judgments, especially from LLMs, may themselves be uncertain rather than binary. One can modify the setup such that the oracle instead returns an inconsistency score $p_{\mathrm{incons}}(S)\in[0,1]$ for a claim subset $S$. \QXR could then apply a threshold $\tau$, while abstaining or re-querying near the decision boundary. This preserves the diagnosis-and-repair architecture while making uncertainty explicit.

Weighted repair and graded judging address complementary questions: \emph{which claims should be preferred} when repairing a localized conflict, and \emph{how certain is the judge} that a conflict exists? A further extension could combine both, allowing uncertainty in conflict judgments to inform the subsequent repair.

\paragraph{Approximate Repair.} \QXR currently localizes a subset-minimal conflict before repair. Both diagnosis and repair can be relaxed to trade preservation quality for fewer oracle calls. Localization can terminate early and return a non-minimal conflict; repair can operate on partial conflict families rather than re-diagnosing to a fixpoint; and scopes can be sampled rather than exhaustively swept, reducing the $m$ factor in $Q=\mathcal O(I\cdot m \cdot k\log N)$. In the uncached case, \QXR’s token cost is at most $\mathcal O(Q)$ times that of a single full-set judgment, although localization queries typically contain fewer claims.

\section{Conclusion}
We introduced global consistency repair for natural-language claim sets: the task of preserving the largest jointly coherent subset of claims while localizing compact sources of inconsistency. This setting captures a failure mode missed by local factuality verification and pairwise contradiction detection, since individually plausible claims may still be globally inconsistent. We instantiated this formulation with \QXR, a diagnosis-and-repair framework that uses LLM consistency judgments to localize conflicting subsets and restore coherence through selective repair. Our analysis shows why set-level reasoning is necessary, and our experiments show that \QXR preserves substantially more valid claims than direct all-at-once judging across factuality and RAG settings. More broadly, these results suggest that factuality in claim aggregation systems should be treated as a preservation-oriented repair problem, not only as local verification.

\section*{Limitations}

Our theoretical analysis uses an idealized abstraction of consistency judging. Assumptions such as monotonicity and perfect judging are used to derive worst-case guarantees, but real LLM judgments may be noisy, prompt-sensitive, or non-monotonic. Our experiments therefore use a single LLM call per query ($r=1$) rather than repeated querying or majority voting. Robustness may benefit from using diverse LLM judges, which we leave to future work.

\QXR also does not avoid the worst-case complexity of global repair: a claim set may contain exponentially many MUSes, and minimum-deletion repair reduces to a hitting-set problem. Structured localization therefore trades additional inference cost for better preservation, which may be expensive for large claim sets. Approximate localization or partial repair may be preferable when inference budget is limited. Finally, our experiments use moderate-sized claim collections constructed from existing benchmarks. Larger and more heterogeneous settings, such as long-context RAG and multi-document synthesis, may contain more varied conflicts and will require benchmarks that better reflect real-world claim collections.

\section*{Ethical Considerations}

\paragraph{Potential Risks.}
\QXR is intended to improve the consistency of claim sets; however, consistency should not be equated with truth. In practice, inconsistencies may reflect legitimate disagreement between sources or varying levels of specificity. Automatically removing claims may therefore suppress valid or minority evidence. This risk also depends on the oracle, which in our setting is an LLM and may inherit model errors or biases.

\paragraph{Use of Generative AI.}
Generative AI tools were used for writing assistance to improve flow, clarity, and reduce repetition, as well as for coding assistance. The authors developed the scientific ideas, technical content, experimental design, and analysis, and reviewed and revised the resulting text.

\bibliography{custom}

\appendix

\section{Proofs and Theoretical Details}
\label{app:proof}

Let $F=\{f_1,\dots,f_N\}$ be a finite fact set, let $C=\{C_1,\dots,C_m\}$ with $C_i\subseteq F$ be a collection of constraint scopes, and let $O:2^F\to\{\mathsf{cons},\mathsf{incons}\}$ be a perfect oracle that returns whether a subset of facts is jointly consistent. Throughout the proofs, we use two standard properties:
\begin{itemize}
    \item \textbf{Monotonicity}: If $U\subseteq S\subseteq F$ and $O(U)=\mathsf{incons}$, then $O(S)=\mathsf{incons}$.
    \item \textbf{Existence of MUS}: If $O(S)=\mathsf{incons}$ and $S$ is finite, then $S$ contains a MUS $U\subseteq S$.
\end{itemize}
These follow from finiteness and the definition of minimality.

\subsection{Objective Equivalence}
\label{appendix:objective}

\begin{theorem}[Objective Equivalence]
\label{lem:equiv_obj}
For all $i\in [m]$, maximizing coverage,
\begin{align}
\max_{F'\subseteq F}\ |F'|
\quad
\text{s.t.}\quad
A(F'\cap C_i)=\mathsf{cons},
\label{eq:max_coverage}
\end{align}
is equivalent to minimizing deletions,
\begin{align}
\min_{R\subseteq F}\ |R|
\quad
\text{s.t.}\quad
A((F\setminus R)\cap C_i)=\mathsf{cons}
\label{eq:min_delete_feasibility}
\end{align}
\end{theorem}

\begin{proof}
Define a bijection between solutions by setting $R=F\setminus F'$ and $F'=F\setminus R$. Then
\begin{align}
|F'|=|F|-|R|.
\end{align}
Hence, maximizing $|F'|$ is equivalent to minimizing $|R|$. Under this bijection, the feasibility constraints in \cref{eq:max_coverage,eq:min_delete_feasibility} are identical.
\end{proof}

\subsection{Reduction to Hitting Set}
\label{appendix:hitting_reduction}

Let $\mathcal{A}_{\mathsf{incons}}$ be the family of all MUSes contained in the scopes $C$. Formally,
\begin{align}
\label{eq:a_incons}
&\mathcal{A}_{\mathsf{incons}}=  \\
&\{\, U \subseteq C_j : j\in[m],\ U \text{ is a MUS w.r.t. } A \,\} \nonumber.
\end{align}

\begin{theorem}[Repair-Hitting Set Equivalence]
\label{thm:hitting_equiv}
A set $R\subseteq F$ is feasible for \cref{eq:min_delete_feasibility} if and only if it is a hitting set for $\mathcal{A}_{\mathsf{incons}}$.
\end{theorem}

\begin{proof}
First, assume $R\subseteq F$ is feasible for the minimum-deletion problem. Suppose for contradiction that there exists $U\in\mathcal{A}_{\mathsf{incons}}$ with $R\cap U=\varnothing$. By definition of $\mathcal{A}_{\mathsf{incons}}$, there is some $i$ with $U\subseteq C_i$, and thus
\begin{align}
U\subseteq (F\setminus R)\cap C_i.
\end{align}
Since $A(U)=\mathsf{incons}$, monotonicity implies that $(F\setminus R)\cap C_i$ is inconsistent, contradicting feasibility of $R$. Hence $R\cap U\neq\varnothing$ for all $U\in\mathcal{A}_{\mathsf{incons}}$, so $R$ is a hitting set.

Conversely, assume $R$ is a hitting set for $\mathcal{A}_{\mathsf{incons}}$, i.e.,
\begin{align}
R\cap U\neq\varnothing
\quad
\forall U\in\mathcal{A}_{\mathsf{incons}}.
\end{align}
Suppose for contradiction that $R$ is not feasible. Then there exists some $i$ such that
\begin{align}
A((F\setminus R)\cap C_i)=\mathsf{incons}.
\end{align}
By the existence-of-MUS property, $(F\setminus R)\cap C_i$ contains a MUS $U\subseteq (F\setminus R)\cap C_i$. Then $U\in\mathcal{A}_{\mathsf{incons}}$, but $U\cap R=\varnothing$, contradicting that $R$ hits all MUSes. Hence $A((F\setminus R)\cap C_i)=\mathsf{cons}$ for all $i$, so $R$ is feasible.
\end{proof}

\subsection{Pairwise Checks Are Insufficient}
\label{appendix:pairwise_insuff}

\begin{restatable}[Pairwise Insufficiency]{theorem}{pairwise}
\label{thm:pairwise_insuff}
For any $N\ge 3$, there exists a fact set $F$ of size $\binom{N}{2}$ such that every pair of facts is jointly consistent, yet $F$ is globally inconsistent.
\end{restatable}

\begin{proof}
Consider a set of variables $x_1,\dots,x_N$ taking values in a finite domain $D$ of labels. We build a constraint graph $G=(V,E)$ where $V=\{1,\dots,N\}$ and an edge $(i,j)\in E$ represents a pairwise incompatibility constraint $x_i\neq x_j$. A labeling $\ell:V\to D$ is consistent if and only if it is a proper coloring of $G$, meaning adjacent vertices receive different labels. Thus, global consistency of the fact set corresponds to $k$-colorability of $G$, where $k=|D|$.

Let $G$ be the complete graph $K_N$ on $N$ vertices, and let the domain have size $|D|=N-1$. For each edge $(i,j)$, include the pairwise constraint $x_i\neq x_j$ as a fact.

We first show that every pair of facts is jointly consistent. Any two facts correspond to two edges in $K_N$. The subgraph induced by two distinct edges is either two disjoint edges or two edges sharing one endpoint. In all cases, this subgraph is bipartite and hence 2-colorable. Since $|D|=N-1\ge 2$, there is an assignment of labels satisfying both facts. Therefore every pair of facts in $F$ is jointly consistent. However, a proper coloring of $K_N$ requires $N$ distinct labels because $\chi(K_N)=N$. Since $|D|=N-1$, the complete graph $K_N$ is not $(N-1)$-colorable. Equivalently, by the pigeonhole principle, $N$ vertices cannot be assigned pairwise distinct labels from a set of only $N-1$ labels. Therefore, the full set of facts is inconsistent.

Hence, for every unordered pair $i<j$, include the claim $f_{ij}=x_i\neq x_j$.  Since there are $\binom{N}{2}$ unordered pairs, the resulting claim set has size $|F|=\binom{N}{2}$.
\end{proof}

For intuition, consider three Boolean variables $A,B,C\in\{0,1\}$ and the three facts $f_1:A\oplus B=1$, $f_2:B\oplus C=1$, and $f_3:C\oplus A=1$. Any two facts are jointly satisfiable, for example $f_1,f_2$ hold for $A=1,B=0,C=1$. But all three cannot hold simultaneously: from $f_1$ and $f_2$ we have $A=\neg B$ and $C=\neg B$, hence $A=C$, contradicting $f_3$ since then $C\oplus A=0$. This is the special case $N=3$ of the graph-based construction.

\subsection{Soundness under a Perfect Oracle}
\label{app:sound}

\begin{theorem}[Soundness under Perfect Oracle]
\label{thm:soundness}
Assume that the LLM oracle $O$ has zero error and matches the ground-truth consistency function $A$. Let $R\subseteq F$ be the set of facts removed by \QXR. Then, for every $i\in[m]$, the retained subset $F'=F\setminus R$ satisfies
\begin{align}
A(F'\cap C_i)=\mathsf{cons}.
\end{align}
\end{theorem}

\begin{proof}
By lines 11-12 of \Cref{alg:main}, after every iteration each stored scope descended from $C_i$ is $C_i \cap F'$ where $F'$ is the current retained fact set. When the algorithm terminates, the condition $\exists\  C_i \in C$ s.t. $O(C_i)=\textsf{incons}$ evaluates to false, hence $O(F'\cap C_i)=\textsf{cons}$ for every $i \in [m]$. Since the oracle is perfect, we therefore have $O\equiv A$ and therefore $A(F'\cap C_i) = \textsf{cons}$ for every $i \in [m]$.
\end{proof}

\subsection{Soundness with Respect to Extracted Conflicts}
\label{app:practical_soundness}

In practical scenarios, the MUSes $\mathcal{U}_{\mathrm{found}}$ extracted by \Cref{alg:main} may contain errors because the procedure depends on a noisy oracle $O$. However, the retained set is still sound with respect to the conflicts actually extracted.

\begin{theorem}[Practical Soundness w.r.t.\ Extracted Conflicts]
\label{thm:practical_soundness}
Let $\mathcal{U}_{\mathrm{found}}$ be the set of conflict sets extracted by \Cref{alg:main} under oracle $O$. Let $R\subseteq F$ satisfy
\begin{align}
R\cap U\neq\varnothing
\quad
\forall U\in\mathcal{U}_{\mathrm{found}}.
\end{align}
Then the retained set $F\setminus R$ contains none of the extracted conflicts in $\mathcal{U}_{\mathrm{found}}$.
\end{theorem}

\begin{proof}
For every $U\in\mathcal{U}_{\mathrm{found}}$, the condition $R\cap U\neq\varnothing$ implies that $U\not\subseteq F\setminus R$. Thus, none of the discovered conflict sets in $\mathcal{U}_{\mathrm{found}}$ remain after removal.
\end{proof}

\subsection{Greedy Hitting Set Approximation}
\label{appendix:greedy_approx}

\begin{restatable}[Greedy Hitting Set Approximation~\citep{vazirani2001approximation}]{lemma}{HittingSetApprox}
\label{lem:hittingset}
Let $\mathcal U$ be a family of $m=|\mathcal U|$ MUSes, or conflict sets, over facts $F$. The greedy hitting-set algorithm returns a hitting set $H$ satisfying
\begin{align}
|H|
\le
(1+\ln m)\,|H^\star|,
\end{align}
where $H^\star$ is an optimal minimum-cardinality hitting set of $\mathcal{U}$.
\end{restatable}

\begin{proof}
The hitting-set problem is dual to set cover. For each fact $f\in F$, define the set
\begin{align}
S_f=\{\, U\in\mathcal{U}: f\in U\,\}.
\end{align}
Choosing facts that hit every conflict set is equivalent to choosing sets $S_f$ that cover every element of $\mathcal{U}$. The standard greedy set-cover analysis gives an approximation factor $H_m\le 1+\ln m$, where $m=|\mathcal{U}|$. Translating back to hitting set gives
\begin{align}
|H|
\le
(1+\ln m)\,|H^\star|.
\end{align}
\end{proof}

\subsection{Information-Theoretic Cost of Conflict Localization}
\label{appendix:localization_lower_bound}

To understand the intrinsic cost of identifying conflicting claims, we isolate the localization step from the repair problem and consider a simplified setting with a single hidden conflict.

\begin{definition}[Hidden-Conflict Query Model]
\label{def:hidden_conflict}
Let $F$ be a fact set with $|F|=N$, and let $U^\star\subseteq F$ be an unknown conflict set of size $|U^\star|=k$. A subset query on $S\subseteq F$ returns
\begin{align}
O^\star(S)=
\begin{cases}
\mathsf{incons} & \text{if } U^\star\subseteq S,\\
\mathsf{cons} & \text{otherwise.}
\end{cases}
\end{align}
The goal is to identify $U^\star$ exactly using adaptive subset queries.
\end{definition}

\begin{theorem}[Lower Bound for Exact Conflict Localization]
\label{thm:lower_bound_localization}
In the hidden-conflict query model, any adaptive algorithm that always identifies $U^\star$ must make at least $\left\lceil
\log_2 \binom{N}{k}
\right\rceil$ queries in the worst case.
\end{theorem}

\begin{proof}
There are exactly $\binom{N}{k}$ possible choices for the hidden conflict $U^\star$. Any deterministic adaptive algorithm can be viewed as a binary decision tree whose internal nodes are oracle queries and whose outgoing edges correspond to the two possible oracle responses, $\mathsf{cons}$ and $\mathsf{incons}$. If the algorithm makes at most $q$ queries in the worst case, then this decision tree has depth at most $q$ and therefore at most $2^q$ leaves.

For the algorithm to identify $U^\star$ exactly for every possible hidden conflict, each of the $\binom{N}{k}$ possibilities must lead to a distinct leaf. Hence $2^q
\ge
\binom{N}{k}.$
Taking logarithms gives $q
\ge
\log_2 \binom{N}{k}.$
Since $q$ is an integer, the worst-case number of queries must be at least $\left\lceil
\log_2 \binom{N}{k}
\right\rceil.$
\end{proof}

\subsection{Cost of Exact Localization}
\label{appendix:localization_cost}

\begin{corollary}[Cost of Exact Localization]
\label{cor:near_opt_qx}
For $1\le k\le N/2$,
\begin{align}
\log \binom{N}{k}
=
\Theta\!\left(k\log(N/k)\right).
\end{align}
Hence any exact deterministic adaptive subset-query algorithm requires $\Omega(k\log(N/k))$ queries in the worst case. The QuickXplain instantiation used in \QXR localizes a conflict of size $k$ using $\mathcal{O}(k\log N)$ consistency queries, matching this lower bound up to logarithmic factors and becoming particularly sharp in the small-conflict regime.
\end{corollary}

\begin{proof}
For $1\le k\le N/2$, standard binomial coefficient bounds give
\begin{align}
\left(\frac{N}{k}\right)^k
\le
\binom{N}{k}
\le
\left(\frac{eN}{k}\right)^k.
\end{align}
Taking logarithms yields
\begin{align}
k\log(N/k)
\le
\log \binom{N}{k}
\le
k\log(eN/k).
\end{align}
Thus $\log\binom{N}{k}=\Theta(k\log(N/k))$. Combining this with \Cref{thm:lower_bound_localization} gives the stated lower bound.
\end{proof}

\subsection{Query Complexity of \QXR}
\label{appendix:query_complexity}

\begin{theorem}[Query Complexity of \Cref{alg:main}]
\label{thm:query_complexity}
Let $N=|F|$ be the number of claims, $m=|C|$ the number of scopes, $k$ the maximum size of any MUS localized by \QXR, and $I$ the number of outer iterations until termination. Assuming \textsc{LocalizeMUS} is instantiated with QuickXplain, \Cref{alg:main} makes at most
\begin{align}
\mathcal{O}\!\left(
I\cdot m\cdot (1+k\log N)
\right)
\end{align}
consistency queries. For $k\ge 1$, this simplifies to
\begin{align}
\mathcal{O}\!\left(
I\cdot m\cdot k\log N
\right).
\end{align}
\end{theorem}

\begin{proof}
At the start of each outer iteration, \QXR checks each of the $m$ scopes to identify the currently inconsistent scopes, incurring $\mathcal{O}(m)$ oracle calls. For each inconsistent scope, it invokes \textsc{LocalizeMUS}. Under the QuickXplain instantiation, localizing a MUS of size at most $k$ within a scope of size at most $N$ requires $\mathcal{O}(k\log N)$ consistency queries. Since at most $m$ scopes are localized in a single outer iteration, each iteration requires
\begin{align}
\mathcal{O}\!\left(m + m k\log N\right)
=
\mathcal{O}\!\left(m(1+k\log N)\right)
\end{align}
oracle calls. Over $I$ outer iterations, the total number of calls is
\begin{align}
\mathcal{O}\!\left(
I\cdot m\cdot (1+k\log N)
\right).
\end{align}
For $k\ge 1$, the term $1+k\log N$ is absorbed into $\mathcal{O}(k\log N)$, yielding $\mathcal{O}(I\cdot m\cdot k\log N)$.
\end{proof}

\section{QuickXplain Pseudocode}
\label{appendix:qx}

QuickXplain \citep{1597148.1597177} as shown in \Cref{alg:quickxplain}, localizes a minimal unsatisfiable subset from an inconsistent set $S$ under a consistency oracle $O:2^F\to\{\mathsf{cons},\mathsf{incons}\}$. The algorithm adaptively queries subsets to find a subset-minimal inconsistent core with logarithmic depth in $|S|$. QuickXplain narrows the inconsistent subset by recursively testing halves of $S$ while conditioning on a background context $B$. For a perfect oracle, QuickXplain requires $\mathcal{O}(k\log|S|)$ queries to locate a conflict of size $k$.

\begin{algorithm}[t]
\caption{QuickXplain (QX) for MUS Extraction}
\label{alg:quickxplain}
\KwIn{Oracle $O$; candidate set $S$; background $B$ assumed consistent}
\KwOut{Subset-minimal inconsistent set $\Delta\subseteq S$ or $\varnothing$}
\If{$S=\varnothing$}{\Return{$\varnothing$}}
\If{$O(B)=\mathsf{incons}$}{\Return{$\varnothing$}}
\If{$O(B\cup S)=\mathsf{cons}$}{\Return{$\varnothing$}}
\If{$|S|=1$}{\Return{$S$}}
Split $S$ into two nearly equal parts $S_1,S_2$\;
$\Delta_1 \gets \textsc{QX}(O,S_1,B\cup S_2)$\;
$\Delta_2 \gets \textsc{QX}(O,S_2,B\cup \Delta_1)$\;
\Return{$\Delta_1\cup \Delta_2$}
\end{algorithm}

\section{Additional Structural Theory}
\label{app:structural_theory}

The main text focuses on the algorithmic behavior of \QXR under adaptive localization and greedy repair. The results in this section address a complementary question: when is exact global consistency repair structurally tractable, independent of the particular \QXR implementation? Worst-case hardness alone does not explain why decomposition-based methods can work well in practice. In many NLP settings, scopes overlap in structured rather than arbitrary ways: for example, claims may be grouped by document, entity, event, retrieved evidence cluster, or local subtask. The scope interaction hypergraph formalizes this overlap structure.

The join-tree results below identify a favorable regime: if scope overlaps are tree-like and global feasibility is captured by scope-wise feasibility, exact repair can be solved by dynamic programming. We also further show that when the conflict graph has bounded treewidth, exact repair is tractable with complexity controlled by that treewidth rather than by the full number of facts. We do not claim that open-ended NLP instances satisfy these assumptions exactly. Rather, they provide a useful lens for understanding why decomposition can be effective when conflicts are localized and scope overlaps are limited.

\subsection{Fact-Disjoint Decomposition}
\label{app:proof_decomposition_opt}

\begin{definition}[Scope Interaction Hypergraph]
Let $F$ be the fact set and $C=\{C_1,\dots,C_m\}$ a family of scopes with $C_i\subseteq F$. The \emph{scope interaction hypergraph} is the hypergraph $\mathcal{H}=(F,C)$ whose vertices are facts and whose hyperedges are scopes.
\end{definition}

\begin{lemma}[Additive Decomposition over Fact-Disjoint Components]
\label{lemm:decoposition_opt}
If the scope interaction hypergraph $\mathcal{H}$ decomposes into fact-disjoint connected components, then the minimum-deletion repair problem decomposes additively across components.
\end{lemma}

\begin{proof}
Let the fact-disjoint connected components of $\mathcal{H}$ be
\begin{align}
(F^{(1)},C^{(1)}),\dots,(F^{(t)},C^{(t)}),
\end{align}
where $F=\dot\cup_{\ell=1}^t F^{(\ell)}$ and every scope in $C^{(\ell)}$ is contained in $F^{(\ell)}$. Here $\dot\cup$ denotes disjoint union.

For any feasible global repair $R\subseteq F$, define $R^{(\ell)}=R\cap F^{(\ell)}$. Since each scope $C_i\in C^{(\ell)}$ lies entirely inside $F^{(\ell)}$, we have
\begin{align}
(F\setminus R)\cap C_i
=
(F^{(\ell)}\setminus R^{(\ell)})\cap C_i.
\end{align}
Hence each $R^{(\ell)}$ is feasible for the repair subproblem on component $\ell$. Therefore $|R^{(\ell)}|\ge \mathrm{OPT}_\ell$ for every $\ell$, where $\mathrm{OPT}_\ell$ denotes the minimum deletion cost in component $\ell$. Summing over components gives
\begin{align}
|R|
=
\sum_{\ell=1}^t |R^{(\ell)}|
\ge
\sum_{\ell=1}^t \mathrm{OPT}_\ell.
\end{align}

Conversely, let $R_\ell^\star$ be an optimal repair for component $\ell$, and define $R^\star
\triangleq
\bigcup_{\ell=1}^t R_\ell^\star.$
Again using that scopes do not cross components, every scope $C_i\in C^{(\ell)}$ satisfies
\begin{align}
(F\setminus R^\star)\cap C_i
=
(F^{(\ell)}\setminus R_\ell^\star)\cap C_i,
\end{align}
which is consistent by feasibility of $R_\ell^\star$. Thus $R^\star$ is a feasible global repair, with
\begin{align}
|R^\star|
=
\sum_{\ell=1}^t |R_\ell^\star|
=
\sum_{\ell=1}^t \mathrm{OPT}_\ell.
\end{align}
Therefore the global optimum decomposes additively across connected components.
\end{proof}

\subsection{Scope Subsumption}
\label{app:scope_sub_proof}

The scoped formulation generalizes global consistency, recovered by $C=\{F\}$. Under monotonicity, nested scopes are redundant: if one scope is contained in another, satisfying the larger scope already guarantees satisfaction of the smaller one. Thus, when $F\in C$, all proper subscopes can be removed without changing the feasible repair set. Scope-locality should therefore be understood as an assumption on the remaining non-redundant scope family.

\begin{lemma}[Scope Subsumption]
\label{lemm:scope_sub}
Assume inconsistency is monotone. If $C_j\subseteq C_i$, then the feasibility constraint on $C_j$ is implied by the feasibility constraint on $C_i$, so $C_j$ may be removed from the scope family without changing the feasible set of repairs.
\end{lemma}

\begin{proof}
Since $C_j\subseteq C_i$, we have
\begin{align}
(F\setminus R)\cap C_j
\subseteq
(F\setminus R)\cap C_i.
\end{align}
Suppose for contradiction that $A((F\setminus R)\cap C_j)=\mathsf{incons}$. By monotonicity of inconsistency, this implies
\begin{align}
A((F\setminus R)\cap C_i)=\mathsf{incons},
\end{align}
contradicting feasibility of $C_i$. Therefore feasibility on $C_i$ implies feasibility on $C_j$.
\end{proof}

\subsection{Exact Repair on Join Trees}
\label{app:proof_exact_repair}

\begin{definition}[Join-Tree Acyclicity]
The scope interaction hypergraph $\mathcal{H}=(F,C)$ is \emph{join-tree acyclic} if there exists a tree $T$ whose nodes are the scopes $C_i\in C$ such that, for every fact $f\in F$, the set of scopes containing $f$ induces a connected subtree of $T$:
\begin{align}
\{\, C_i\in C : f\in C_i\,\}.
\end{align}
\end{definition}

\begin{definition}[Scope Locality]
We say the repair problem is \emph{scope-local} if global feasibility is fully captured by scope-wise feasibility. Equivalently, every global inconsistency contains a minimal unsatisfiable subset fully contained in some scope $C_i$.
\end{definition}

Join-tree acyclicity is useful because it gives a principled version of ``structured decomposition.'' The connected-subtree condition ensures that if a fact appears in multiple scopes, all occurrences of that fact can be coordinated through a separator in the tree. This allows dynamic programming to combine locally feasible repairs without making inconsistent keep/delete decisions for shared facts. Scope-locality is not required for the dynamic program to optimize the scope-wise objective in \Cref{eq:min_delete_feasibility}; rather, it ensures that feasibility under these scope-wise constraints coincides with global consistency.

\begin{theorem}[Exact Repair on Join Trees]
\label{thm:exact_repair}
Assume:
\begin{enumerate}[leftmargin=1.5em]
    \item the repair problem is scope-local;
    \item the scope interaction hypergraph $\mathcal{H}$ is join-tree acyclic.
\end{enumerate}
Then the minimum-deletion repair problem
\begin{align}
\min_{R\subseteq F}\ |R|
\quad
\text{s.t.}\quad
A((F\setminus R)\cap C_i)=\mathsf{cons},
\end{align}
$\forall i\in[m]$, is solvable exactly by dynamic programming over a join tree of $\mathcal{H}$.
\end{theorem}

\begin{proof}
Let $T$ be a join tree of $\mathcal{H}$ whose nodes are the scopes $C_i$. Root $T$ at an arbitrary scope $C_r$. For each non-root node $i$, let $\mathrm{Pa}(i)$ denote its parent and let $S_i
\triangleq
C_i\cap C_{\mathrm{Pa}(i)}$ 
be the separator between $i$ and its parent. For the root, define $S_r=\varnothing$.

For each node $i$, let $T_i$ denote the subtree rooted at $i$, and let $F_i^\downarrow \triangleq  \bigcup_{C_j\in T_i} C_j$ be the set of facts appearing in that subtree.

We define a dynamic program over $T$. A state at node $i$ is a keep/delete assignment $\sigma$ on the separator $S_i$, specifying which facts in $S_i$ are retained. Let $\mathrm{DP}_i(\sigma)$ denote the minimum number of deletions inside $F_i^\downarrow\setminus S_i$ such that:
\begin{enumerate}
    \item the resulting retained facts in the subtree agree with $\sigma$ on $S_i$;
    \item every scope in the subtree is locally consistent; and
    \item the retained sets chosen for adjacent scopes agree on their overlaps.
\end{enumerate}

Fix a node $i$ and a separator assignment $\sigma$ on $S_i$. Any feasible solution for the subtree rooted at $i$ must choose a retained subset on the facts of $C_i$ that extends $\sigma$. Equivalently, it chooses a keep/delete assignment $\tau$ on $C_i$ whose restriction to $S_i$ equals $\sigma$, and such that the retained subset of $C_i$ is locally consistent. For each child $j$ of $i$, this assignment induces a separator assignment $\tau|_{C_i\cap C_j}$ for the child subtree. Hence the optimal value satisfies $\mathrm{DP}_i(\sigma)
=
\min_{\substack{
\tau \text{ on } C_i\\
\tau|_{S_i}=\sigma\\
\tau \text{ feasible}
}}
\left\{
\mathrm{\kappa}
\right\},$
where 
\begin{align}
\label{eq:dp_cost}
    \kappa = \mathrm{cost}_i(\tau)
+
\sum_{j\in\mathrm{child}(i)}
\mathrm{DP}_j\!\left(\tau|_{C_i\cap C_j}\right)
\end{align}
and $\mathrm{cost}_i(\tau)$ counts deletions in $C_i\setminus S_i$. 

We now argue correctness. First, any solution assembled by the recurrence is globally feasible. By construction, each scope is locally feasible, and adjacent scopes agree on shared facts. Since $\mathcal{H}$ is join-tree acyclic, the scopes containing any fact form a connected subtree, so these pairwise agreements define a unique global keep/delete decision for every fact. Therefore the dynamic program constructs a well-defined retained set $X\subseteq F$. Because the problem is scope-local, local feasibility of every scope implies global feasibility of $X$.

Conversely, let $X^\star\subseteq F$ be any globally feasible retained set. It induces, for every node $i$, a separator assignment on $S_i$ and a local keep/delete assignment on $C_i$. Since $X^\star$ is globally feasible and inconsistency is monotone, $X^* \cap C_i \subseteq X^*$ implies every scope is locally feasible. These induced assignments therefore satisfy the dynamic-program constraints and achieve exactly the deletion cost of $X^\star$. Hence every globally feasible repair is represented by some valid collection of DP choices.

Thus the dynamic program is both sound and complete, and its optimum value equals the exact minimum-deletion repair cost.
\end{proof}

\subsection{Tractability under Bounded Conflict-Graph Treewidth}

The complexity of exact repair could also be characterized by the interaction structure among minimal conflicts. Let $\mathcal A_{\mathrm{incons}}$ be as defined in \Cref{eq:a_incons} (the family of scope-contained MUSes). 

\begin{definition}[Tree Decomposition and Treewidth]
Let $G=(V,E)$ be a graph. A \emph{tree decomposition} of $G$ is a tree $T$ whose nodes $x\in V(T)$ are associated with bags $B_x\subseteq V$, such that every vertex and edge of $G$ is contained in some bag and, for every $v\in V$, the bags containing $v$ form a connected subtree of $T$. The width of the decomposition is $\max_{x\in V(T)}|B_x|-1$, and the \emph{treewidth} $\operatorname{tw}(G)$ is the minimum width over all tree decompositions of $G$.
\end{definition}

\begin{definition}[Conflict Graph]
    The \emph{conflict graph} induced by $\mathcal A_{\mathrm{incons}}$ is $G_{\mathcal A}=(F, E)$ where $\{f,f'\}\in E
    \iff
    (f\neq f'\land \exists U\in\mathcal{A}_{\mathrm{incons}}
    \text{ such that } f,f'\in U)$.
\end{definition}
Under this definition, every conflict $U \in \mathcal A_{\mathrm{incons}}$ forms a clique in $G_\mathcal{A}$. \Cref{fig:tree-decomposition} illustrates the above concepts.

\begin{theorem}[Exact Repair under Bounded Conflict-Graph Treewidth]
\label{thm:bounded_ga_treewidth}
Suppose we are given $\mathcal A_{\mathrm{incons}}$, and let $t=\operatorname{tw}(G_\mathcal{A})$ denote the treewidth of $G_\mathcal{A}$. Given a tree decomposition of $G_\mathcal{A}$ of width $t$, the minimum-deletion repair problem in \Cref{eq:min_delete_feasibility} can be solved exactly in $2^{t+1}\cdot \operatorname{poly}(|F|+|\mathcal A_{\mathrm{incons}}|)$ time.
\end{theorem}
\begin{proof}
    By \Cref{thm:hitting_equiv}, a deletion set $R\subseteq F$ is feasible for \Cref{eq:min_delete_feasibility} if and only if $R\cap U \neq \varnothing$ for every $U \in \mathcal{A}_\mathrm{incons}$. Therefore, exact minimum-deletion repair is equivalent to minimum hitting set over $\mathcal{A}_\mathrm{incons}$.

    First, we show every conflict lies within a single bag. Under the definition of $G_\mathcal{A}$, there exists an edge between every pair of distinct elements of $U$, therefore, a MUS $U$ induces a clique. Now, suppose $U$ forms a $|U|$-clique, for each vertex $f_i \in U$, consider the set of tree nodes whose bags contain it: $T_i=\{x \in V(T)\mid f_i \in B_x\}$. By the running-intersection property, each $T_i$ is a connected subtree of $T$. Because $U$ is a $|U|$-clique, every pair of distinct vertices  $f_i,f_j\in U$ satisfies $\{f_i,f_j\}\in E(G_\mathcal{A})$.  By edge coverage, some bag contains both $f_i$ and $f_j$, and hence  $T_i\cap T_j\neq\varnothing$. Thus, the subtrees  $T_1,\ldots,T_{|U|}$ are pairwise intersecting and subtrees of a tree satisfy the Helly property   and hence $\bigcap_{i=1}^{|U|} T_i\neq\varnothing$.  Therefore, there exists some $x\in V(T)$ such that  $U\subseteq B_x$. Assign each  $U\in\mathcal{A}_\mathrm{incons}$ to one such bag.

    We now perform dynamic programming over the given bag-tree decomposition. We may assume the decomposition is a nice tree decomposition of the same width and polynomial size. For each bag $B_x$, a state specifies the keep/delete assignment of the facts in $B_x$. Since $|B_x|\leq t+1$, there are at most $2^{t+1}$ states per bag. A state is only feasible if it deletes at least one fact from every conflict assigned to that bag. Transitions between adjacent bags enforce agreement of their shared facts and accumulate deletion costs without double counting.

    The minimum-cost feasible state at the root therefore yields a minimum hitting set of $\mathcal A_\mathrm{incons}$ and by \Cref{thm:hitting_equiv}, an exact minimum-deletion repair. Therefore, our dynamic program has at most $2^{t+1}$ states per bag, with polynomial time for additional checks per bag, giving total running time $2^{t+1}\cdot \operatorname{poly}(|F|+|\mathcal A_{\mathrm{incons}}|)$. 
\end{proof}

Taken together, these structural results identify a favorable regime for exact repair. \Cref{lemm:decoposition_opt} shows that independent components can be solved separately, while \Cref{thm:exact_repair,thm:bounded_ga_treewidth} show that sparse, tree-like conflict interactions admit dynamic programming with complexity controlled by conflict-graph treewidth rather than the full number of facts. Thus, when inconsistencies are localized and conflict interactions are limited, global consistency repair need not exhibit worst-case combinatorial behavior.

The framework also extends directly to weighted repair as mentioned in \Cref{sec:qxr_ext}. In many NLP applications, facts carry non-uniform importance due to confidence, provenance, or task-specific utility. Assigning each fact a deletion cost and minimizing total removed weight rather than cardinality preserves the same connection to weighted hitting set, and the join-tree dynamic program carries over by replacing deletion counts with deletion weights.
\begin{figure}[!t]
\centering

\begin{tikzpicture}[
    vertex/.style={
        circle,
        fill,
        inner sep=1.5pt
    },
    bag/.style={
        draw,
        rounded corners=1pt,
        inner xsep=5pt,
        inner ysep=3pt
    },
    every node/.style={font=\small}
]

\begin{scope}[yshift=0cm]

\node[font=\bfseries] at (0,2.75)
    {(a) Conflict Graph and Bag};

\node[vertex] (a) at (-1.05,1.35) {};
\node[vertex] (b) at ( 0.00,1.85) {};
\node[vertex] (c) at ( 1.05,1.35) {};
\node[vertex] (d) at ( 0.00,0.65) {};
\node[vertex] (e) at ( 0.85,0.20) {};

\node[left=3pt]  at (a) {$a$};
\node[above=3pt] at (b) {$b$};
\node[right=3pt] at (c) {$c$};
\node[below=1pt] at (d) {$d$};
\node[right=3pt] at (e) {$e$};

\draw (a)--(b);
\draw (b)--(c);
\draw (a)--(d);
\draw (b)--(d);
\draw (c)--(d);
\draw (c)--(e);
\draw (d)--(e);

\draw[
    dashed,
    rounded corners=4pt
]
    (-1.65,0.20) rectangle (0.45,2.35);

\node at (0,-0.35)
    {$G_{\mathcal A},\qquad B=\{a,b,d\}$};

\end{scope}

\begin{scope}[yshift=-3.8cm]

\node[font=\bfseries] at (0,2.2)
    {(b) Bag Tree};

\node[bag] (B1) at (0,1.45)
    {$B_1=\{a,b,d\}$};

\node[bag] (B2) at (0,0.55)
    {$B_2=\{b,\mathbf{c},d\}$};

\node[bag] (B3) at (0,-0.35)
    {$B_3=\{\mathbf{c},d,e\}$};

\draw (B1) -- (B2);
\draw (B2) -- (B3);

\node[
    text width=6.5cm,
    align=center
] at (0,-1.05)
    {Bags containing $c$ form a connected subtree.};

\end{scope}

\begin{scope}[yshift=-7.8cm]

\node[font=\bfseries] at (0,2.25)
    {(c) Conflict Contained in a Bag};

\node[vertex] (cb) at (-0.9,1.00) {};
\node[vertex] (cc) at (0,1.48) {};
\node[vertex] (cd) at (0.9,1.00) {};

\node[left=3pt]  at (cb) {$b$};
\node[above=2pt] at (cc) {$c$};
\node[right=3pt] at (cd) {$d$};

\draw (cb)--(cc);
\draw (cc)--(cd);
\draw (cb)--(cd);

\node at (0,0.45)
    {$U=\{b,c,d\}$};

\draw[->, thick]
    (0,0.20) -- (0,-0.18);

\node[bag] (MUSbag) at (0,-0.58)
    {$B_2=\{b,c,d\}$};

\node[
    text width=6.5cm,
    align=center
] at (0,-1.28)
    {Every MUS forms a clique and is contained in some bag.};

\end{scope}

\end{tikzpicture}

\caption{
Illustration of bounded conflict-graph treewidth.
(a) A bag contains a subset of conflict-graph vertices.
(b) Bags form a tree satisfying the running-intersection property.
(c) Every MUS induces a clique in $G_{\mathcal A}$ and is therefore
contained in some bag.
Here
$\mathcal{A}_{\mathrm{incons}}
=\{\{a,b,d\},\{b,c,d\},\{c,d,e\}\}$.
}
\label{fig:tree-decomposition}

\end{figure}

\section{Complexity of Pairwise/NLI Baselines}
\label{app:nli_baselines}

Methods that decide consistency by evaluating all sentence pairs require $\binom{N}{2}$ oracle calls. For typical fact-set sizes, e.g., $N\in[30,100]$, that implies $435$--$4950$ calls per instance. Under noisy LLM oracles, each decision may further require $r$ repetitions. In contrast, \QXR performs at most $\mathcal{O}(I\cdot m\cdot k\log N)$ queries, where $k$ is the maximum MUS size, $m$ is the number of scopes, and $I$ is the number of outer rounds. In our experiments, $m=1$ and conflicts are small. \Cref{tab:nli_baselines} summarizes such query complexities.
We use direct all-at-once LLM judging as the primary baseline because it is the strongest natural black-box alternative to \QXR for the same preservation-oriented task. Given the full claim set, the baseline is prompted to return the largest jointly coherent subset, matching the input-output format and objective of \QXR. This makes the comparison conservative: both methods rely on the same LLM judge and the same natural-language claim representation, but \QXR adds explicit conflict localization and hitting-set repair. Pairwise NLI or contradiction-detection baselines are not directly comparable, since pairwise consistency does not certify global consistency and such methods do not by themselves return a maximal coherent subset. Exhaustive subset enumeration or exact MUS search is combinatorially infeasible at realistic claim-set sizes, while symbolic solvers require a formalization of natural-language claims that is unavailable in the black-box setting. Thus, direct LLM judging is the appropriate primary baseline for testing whether structured diagnosis-and-repair improves over the most natural end-to-end alternative.
\begin{table*}[h]
\centering
\small
\begin{tabular}{p{3.2cm}p{7.5cm}c}
\toprule
\textbf{Method} & \textbf{Description} & \textbf{Complexity} \\
\midrule
Pairwise NLI Graph & Run NLI on every pair of claims, removing any node involved in a contradiction edge. & $\mathcal{O}(N^2)$ \\
Transitive Closure / Entailment Graph & Build a full entailment--contradiction graph and perform reasoning, e.g., SAT solving or closure. & $>\mathcal{O}(N^2)$ \\
LLM-as-NLI & Prompt an LLM with two sentences, e.g., ``Does A contradict B?'' & $\mathcal{O}(N^2)$ \\
Multi-premise NLI & Treat the entire set of claims as premises and ask if they are jointly consistent. & $\mathcal{O}(1)$ \\
Clustered Pairwise & For each claim, compare only to its top-$K$ nearest neighbors. & $\mathcal{O}(NK)$ \\
\bottomrule
\end{tabular}
\caption{Common NLI-style baselines for consistency checking and their computational complexity. Pairwise and entailment-graph methods grow quadratically with the number of claims, while multi-premise NLI corresponds to the direct all-at-once baseline.}
\label{tab:nli_baselines}
\end{table*}

\section{Prompts}
\label{app:prompts}

All experiments use publicly available LLM APIs and standard benchmark datasets (e.g., VitaminC) under their existing licenses. This appendix lists the exact prompts used in our experiments. All prompts were intentionally kept simple and symmetric across methods to isolate algorithmic effects rather than prompt engineering. For all baselines that output a subset of facts, models are required to return the full text of each retained fact, not indices, as a Python list enclosed in an \texttt{<answer>} tag.

\begin{table*}[t]
\centering
\setlength{\tabcolsep}{1pt}
\begin{tabular}{lccc ccc ccc ccc ccc}
\toprule
& \multicolumn{3}{c}{Claude-3.7}
& \multicolumn{3}{c}{Claude-4}
& \multicolumn{3}{c}{DeepSeek-R1}
& \multicolumn{3}{c}{GPT-OSS-120B}
& \multicolumn{3}{c}{Mistral Large 2} \\
\cmidrule(lr){2-4}
\cmidrule(lr){5-7}
\cmidrule(lr){8-10}
\cmidrule(lr){11-13}
\cmidrule(lr){14-16}
Prompting Style
& P & R & F1
& P & R & F1
& P & R & F1
& P & R & F1
& P & R & F1 \\
\midrule
Zero-shot
& .986 & .681 & .806
& .989 & .928 & .958
& .993 & .860 & .922
& .994 & .876 & .931
& .947 & .568 & .710 \\
Chain-of-Thought
& .973 & .684 & .803
& .987 & .908 & .946
& .996 & .850 & .917
& .995 & .851 & .917
& .997 & .570 & .725 \\
Decomposition
& .987 & .626 & .766
& .677 & .608 & .641
& .996 & .882 & .936
& .954 & .816 & .880
& .991 & .535 & .695 \\
Few-shot
& .975 & .658 & .786
& .992 & .891 & .939
& .882 & .761 & .817
& .997 & .896 & .944
& .238 & .142 & .178 \\
Self-consistency
& .986 & .547 & .704
& .992 & .768 & .866
& .987 & .718 & .831
& .996 & .475 & .643
& .990 & .522 & .684 \\
\midrule
Original Direct
& .979 & .854 & .912
& .956 & .877 & .915
& .980 & .730 & .837
& .984 & .926 & .953
& .955 & .603 & .739 \\
\QXR
& .956 & .975 & \textbf{.965}
& .938 & .983 & \textbf{.960}
& .973 & .990 & \textbf{.981}
& .956 & .995 & \textbf{.975}
& .968 & .978 & \textbf{.972} \\
\bottomrule
\end{tabular}
\caption{Prompt-sensitivity results across models. We compare direct
judging under alternative prompting strategies with the original direct
baseline and \QXR.}
\label{tab:prompt-sensitivity}
\end{table*}

\paragraph{Subset-consistency oracle (\QXR).}
\begin{quote}\small
\texttt{Factual statements. Some may contradict.}\\
\\
\texttt{\{bg\}}\\
\texttt{Statements:}\\
\texttt{\{facts\_block\}}\\
\\
\texttt{Respond ONLY:}\\
\texttt{- CONSISTENT}\\
\texttt{- INCONSISTENT}\\
\\
\texttt{Answer:}
\end{quote}

\paragraph{Direct baseline (zero-shot).}
\begin{quote}\small
\texttt{Given the following factual statements, some may contradict.}\\
\texttt{Return a Python list of facts that are mutually consistent, meaning all returned facts can be true at the same time.}\\
\\
\texttt{Facts:}\\
\texttt{\{facts\_block\}}\\
\\
\texttt{CRITICAL: Return ONLY a Python list using the FULL TEXT of each fact.}\\
\texttt{<answer>["fact1", "fact2", ...]</answer>}
\end{quote}

\paragraph{Direct baseline (Chain-of-Thought prompting).}
\begin{quote}\small
\texttt{Given the following facts, some may contradict. Find all mutually consistent facts.}\\
\\
\texttt{Facts:}\\
\texttt{\{facts\_block\}}\\
\\
\texttt{Think step-by-step:}\\
\texttt{1. Identify pairs of facts that contradict each other.}\\
\texttt{2. For each contradiction, decide which fact to keep.}\\
\texttt{3. Return the consistent subset.}\\
\\
\texttt{CRITICAL: In the <answer> tag, return the FULL TEXT of each fact, NOT numbers.}\\
\texttt{Example: <answer>["The sky is blue", "Grass is green"]</answer>}
\end{quote}

\paragraph{Direct baseline (Decomposition prompting).}
\begin{quote}\small
\texttt{Task: Select a mutually consistent subset of the following facts.}\\
\\
\texttt{Facts:}\\
\texttt{\{facts\_block\}}\\
\\
\texttt{Step 1 -- Identify contradicting facts.}\\
\texttt{Step 2 -- Decide which facts to keep.}\\
\texttt{Step 3 -- Output the consistent subset.}\\
\\
\texttt{CRITICAL: Return ONLY a Python list with the FULL TEXT of each fact.}\\
\texttt{<answer>["full fact text 1", "full fact text 2", ...]</answer>}
\end{quote}

\paragraph{Direct baseline (Few-shot prompting).}
\begin{quote}\small
\texttt{Given facts that may contradict, return the subset that is mutually consistent.}\\
\\
\texttt{Example 1:}\\
\texttt{- The sky is blue}\\
\texttt{- The sky is red}\\
\texttt{- Grass is green}\\
\texttt{<answer>["The sky is blue", "Grass is green"]</answer>}\\
\\
\texttt{Example 2:}\\
\texttt{- Paris is in France}\\
\texttt{- Paris has 2 million people}\\
\texttt{- Paris has 10 million people}\\
\texttt{<answer>["Paris is in France", "Paris has 2 million people"]</answer>}\\
\\
\texttt{Now solve:}\\
\texttt{Facts:}\\
\texttt{\{facts\_block\}}\\
\\
\texttt{<answer>}
\end{quote}

\paragraph{Direct baseline (Self-consistency prompting).}
\begin{quote}\small
\texttt{Given the following facts, some may contradict.}\\
\texttt{Return ONLY a Python list of facts that are mutually consistent.}\\
\\
\texttt{Facts:}\\
\texttt{\{facts\_block\}}\\
\\
\texttt{<answer>["fact1", "fact2", ...]</answer>}
\end{quote}

\section{Prompting Analysis}
\label{app:prompt}

All main experiments use a zero-shot LLM-judge setting. To assess sensitivity to prompt design, we additionally evaluated the direct-judge baseline on VitaminC using several advanced prompting strategies. Results are shown in \cref{tab:prompt-sensitivity}. Across prompt styles, we observe the same qualitative trend: high precision but substantially lower recall due to over-removal of consistent facts. These experiments were run on a subset of the dataset; for reference, we also restate the original zero-shot baseline and \QXR results from the full evaluation.

\end{document}